\documentclass{article}

\usepackage[preprint, nonatbib]{neurips_2025}

\usepackage[utf8]{inputenc}
\usepackage[T1]{fontenc}
\usepackage{hyperref}
\hypersetup{colorlinks=true, citecolor=blue, linkcolor=blue,
            urlcolor=blue}
\usepackage{url}
\usepackage{booktabs}
\usepackage{amsfonts}
\usepackage{microtype}
\usepackage{xcolor}

\usepackage{amsmath,amssymb,amsthm}
\usepackage{mathrsfs}          
\usepackage{graphicx}

\theoremstyle{plain}
\newtheorem{theorem}{Theorem}[section]
\newtheorem{proposition}[theorem]{Proposition}
\newtheorem{lemma}[theorem]{Lemma}
\newtheorem{corollary}[theorem]{Corollary}
\theoremstyle{definition}

\newtheorem{remark}[theorem]{Remark}
\numberwithin{equation}{section}

\newcommand{\R}{\mathbb{R}}
\newcommand{\SO}{\mathrm{SO}}
\newcommand{\Ort}{\mathrm{O}}
\newcommand{\Stab}{\mathrm{Stab}}
\newcommand{\Hol}{\mathrm{Hol}}
\newcommand{\sgn}{\mathrm{sign}}
\newcommand{\Rmin}{\widetilde{\mathcal{R}}}
\newcommand{\A}{\mathscr{A}}
\newcommand{\Uc}{\mathcal{U}}
\newcommand{\eps}{\varepsilon}

\title{Linguistic Holonomy and Statistical Watermarks:\\
Inner Geometry of Meaning-Preserving Transformations}

\author{%
  Daniele Corradetti \\
  Grupo de F\'\i sica Matem\'atica, Instituto Superior T\'ecnico \\
  Av.\ Rovisco Pais, 1049-001 Lisboa, Portugal \\
  Departamento de Matem\'atica, Universidade do Algarve \\
  Campus de Gambelas, 8005-139 Faro, Portugal \\
  \texttt{danielecorradetti@tecnico.ulisboa.pt} \\
}

\begin{document}

\maketitle

\begin{abstract}
Statistical watermarks for language models live in the freedom of the
signifier: they choose among tokens that are nearly equivalent in
meaning, and they are therefore eroded by exactly those transformations
which move the form of a text while leaving its content in place. The
literature measures such transformations by their endpoint, through the
semantic similarity between the original and the rewritten text. We show
that the endpoint is the wrong statistic. Adapting the formalism of
linguistic loops, we prove that the invariant of a chain of
meaning-preserving transformations factorises canonically into an
endpoint part and a holonomy in the stabiliser of the initial state, the
second of which the semantic deficit cannot see; the loop rotation is
parallel transport on the unit sphere of the embedding space, so that
the analogy with the Wilson loop becomes a theorem rather than a figure
of speech. On the side of the detector we prove an exact identity: the
residual statistic is proportional to the number of positions whose
seeding window survived intact, from which the decay law $\rho^{h+1}$
follows as the independent-edit corollary. The identity has a
disconcerting consequence, which we confirm to three decimal places: at
one and the same retention rate the surviving signal may be one half of
the original, one quarter of it, or exactly nothing, according only to
where the edits fall.
\end{abstract}

\section{Introduction and Motivation}

When Raymond Queneau published his \emph{Exercices de style} in 1947,
inspired by Bach's \emph{Art of Fugue}, he wrote the same trivial
anecdote ninety-nine times. A man is jostled on a bus; two hours later a
friend advises him to move a button on his overcoat. Nothing else
happens. What varies is everything else: the register, the tense, the
person, the alphabet, the sonnet form, the language of a
mathematician. The content is held fixed by fiat, and the form is
allowed to move as far as it can. Queneau took himself to be doing
something intrinsically geometric, and named an early version of the
work \emph{Dod\'eca\`edre}; the intuition, as was argued in \cite{CM25},
is not casual at all.

Eighty years later the same operation has acquired an adversarial use.
Since 2 August 2026 the outputs of at least one major language model
carry a machine-readable mark, in response to the transparency
obligations of Article 50 of the European AI Act \cite{Anthropic26}. The
published description of such marks is careful: the mark does not change
meaning, quality or readability, it survives copying and light editing,
and it is stated to be lost under paraphrase, heavy editing, translation
and mixing with other text. Whoever wishes to remove the mark, then, is
being invited to perform an \emph{Exercice de style}. The question of
how much of the mark survives such an exercise is the subject of this
article.

\subsection*{The landscape}

The dominant construction is the \emph{green-list} watermark of
Kirchenbauer \emph{et al.} \cite{KGW23}: at each step a
pseudo-random subset of the vocabulary, of relative size $\gamma$, is
declared green on the basis of a secret key and of the $h$ preceding
tokens, and the logits of green tokens are raised by a bias $\delta_b$.
Detection counts green tokens and reads off a $z$-score. Setting $h=0$
gives the Unigram scheme of Zhao \emph{et al.} \cite{Zhao24}, whose
robustness is provable. A different family, initiated by Aaronson and
developed by Kuditipudi \emph{et al.} \cite{Kud24} and by Christ, Gunn
and Zamir \cite{CGZ24}, leaves the output distribution mathematically
untouched and replaces instead the source of randomness in the sampler;
these are the \emph{distortion-free} schemes. A production deployment at
scale, using tournament sampling, is described in \cite{Dath24}. A
synoptic view is given in Table~\ref{tab:families}.

\begin{table}[t]
\centering
\small
\begin{tabular}{llccl}
\toprule
Family & Mechanism & Context $h$ & Output law & What an edit destroys \\
\midrule
Green list \cite{KGW23} & logit bias & $\ge 1$ & perturbed &
  token \emph{and} its $h$ predecessors \\
Unigram \cite{Zhao24} & logit bias & $0$ & perturbed &
  the token alone \\
Exponential \cite{Kud24,CGZ24} & sampling rule & $\ge 0$ & exact &
  token \emph{and} seeding context \\
Tournament \cite{Dath24} & sampling rule & $\ge 1$ & exact &
  token \emph{and} seeding context \\
Semantic \cite{Hou24a} & sentence partition & --- & perturbed &
  the sentence embedding \\
\bottomrule
\end{tabular}
\caption{\emph{Synoptic table of the families of text watermarks. The
last column is the one that matters for this article: in every family
except the Unigram scheme, an edit damages not only the position it
touches but every position that used it as seeding context. It is this
asymmetry which the geometry of Section~\ref{sec:holonomy} and the
identity of Section~\ref{sec:watermark} make precise.}}
\label{tab:families}
\end{table}

Against these constructions stands a literature of attacks. Krishna
\emph{et al.} \cite{Kri23} introduced a dedicated paraphraser; Sadasivan
\emph{et al.} \cite{Sad23} iterated it; a recent large-scale study
\cite{Chain26} chains rewritings up to five deep and reports detection
falling from $87.9\%$ on the original outputs to $4.86\%$ after five
hops. Kirchenbauer \emph{et al.} \cite{KGW24} had already observed, on
empirical grounds, that a longer seeding context makes a watermark less
robust to editing.

All of these studies share a habit of measurement. An attack is scored
by its \emph{endpoint}: by the semantic similarity between the original
and the final text, and by a count of how many tokens or hops separate
them. The intermediate states are treated as scaffolding.

\subsection*{The gap}

This habit is not innocent. Consider two chains of rewritings that begin
at the same text and end at texts of identical meaning. One proceeds
directly; the other wanders through three languages and returns. By any
endpoint measure they are the same attack. Are they?

To ask the question one needs a formalism in which the intermediate
states are part of the object rather than part of the apparatus. Such a
formalism exists. In \cite{CM25} a \emph{linguistic loop} was defined as
a chain of transformations that move the signifier while preserving the
semantic core, its \emph{semantic deficit} was defined as the cosine
distance between the first and last embedded states, and a rotation
$\mathrm{R}_\Uc \in \SO(n)$ was associated with the entire chain by
composing the minimal rotations between consecutive states. The
signature of a quadratic form built from $\mathrm{R}_\Uc$ was proposed
as the invariant of the loop, and the construction was compared, in the
introduction of that paper, with the Wilson loop of lattice gauge
theory.

We take up that formalism, and we find three things.

The first is that the proposed invariant collapses. The representing
matrix $I_n - \mathrm{R}^{*}_\Uc$ is always positive semidefinite, so its
Sylvester signature is $(2p,0,n-2p)$ and carries no more information
than its rank. The negative index, which is where a signature normally
keeps its content, is identically zero.

The second is that the analogy with the Wilson loop is not an analogy.
The minimal rotation of \cite{CM25} \emph{is} parallel transport along
the minimising geodesic of the unit sphere; $\mathrm{R}_\Uc$ is parallel
transport along the geodesic polygon of the chain; and the quantity that
the semantic deficit discards is precisely a holonomy, an element of
$\SO(n-1)$ fixing the initial state. Endpoint and path separate
canonically, and they are independent: neither constrains the other.

The third is that on the side of the detector there is an exact law, and
it is not the law the field has been assuming. The residual statistic
after an attack is proportional to the number of scored positions whose
entire seeding window survived. When the edits are independent this
gives $\rho^{h+1}$, which explains in closed form the trend reported in
\cite{KGW24}. When they are not --- and real edits never are --- the
retention rate $\rho$ does not determine the residual at all.

\subsection*{Contribution}

In this work we prove that the invariant of a linguistic loop
factorises canonically as $\mathrm{R}_\Uc = \mathrm{R}_{\mathrm{dir}} H$
with $H$ in the stabiliser of the base point, that the semantic deficit
depends only on the first factor, and that the two factors are
functionally independent; we identify $H$ as the Riemannian holonomy of
the geodesic polygon of the chain, which on the two-sphere is the area
of the enclosed triangle; we show that the signature invariant of
\cite{CM25} degenerates to a rank and is superseded by the angle
spectrum; and we prove the intact-window identity for green-list and
exponential detectors, with its corollary that an adversary who knows
the context width $h$ can annihilate the signal while retaining a
fraction $1 - 1/(h+1)$ of the tokens. Every statement is verified
numerically against exact detectors whose keys we hold, and the
geometric machinery is applied to real chains of round-trip machine
translation.

In Section~\ref{sec:loops} we recall, in a self-contained form, the
apparatus of linguistic loops. In Section~\ref{sec:signature} we prove
the degeneracy of the signature. Section~\ref{sec:holonomy} contains the
factorisation, the identification with parallel transport and the
Gauss--Bonnet corollary. Section~\ref{sec:watermark} turns to the
detector and proves the intact-window law. Section~\ref{sec:experiments}
reports the experiments, including a preregistered test whose outcome we
report whichever way it fell. Section~\ref{sec:conclusions} discusses
what the results do and do not license.

\section{Linguistic loops and their invariants}\label{sec:loops}

We recall the apparatus of \cite{CM25}, in the form in which we shall
use it. The reader who knows that paper may skip to
Section~\ref{sec:signature}; the reader who does not should find here
everything the sequel requires.

Language is not a metric space, but its images under an embedding map
are. A \emph{metrizable linguistic space} is a triple
$(\A,\psi,d)$ where $\A$ is a set of linguistic elements --- words,
phrases, whole propositions --- in a context of interest, and
\begin{equation}\label{eq:embed}
\psi : \A \longrightarrow \R^n , \qquad \lambda \longmapsto \psi(\lambda),
\end{equation}
is an embedding map, generally neither injective nor surjective. A
distance $d_*$ on $\R^n$ induces a \emph{semantic distance} on $\A$ by
composition,
\begin{equation}\label{eq:semdist}
d(\lambda,\nu) := d_*\bigl(\psi(\lambda),\psi(\nu)\bigr),
\end{equation}
and throughout this article $d_*$ is the cosine distance
\begin{equation}\label{eq:cosine}
d_{*,\cos}(x,y) = 1 - \frac{x\cdot y}{\|x\|\,\|y\|}.
\end{equation}
The choice is not innocuous and we shall return to it in
Remark~\ref{rem:group}: it is what makes the whole construction
orthogonally natural rather than linearly natural.

A \emph{linguistic transformation} $U$ is a map of $\A$ into itself which
acts on the signifier while preserving, in a controlled way, the
semantic core. In \cite{CM25} three requirements are imposed: closure,
$U(\lambda) \in \A$; reversibility, the existence of an approximate
inverse $U^{-1}$ with $d(U^{-1}(U(\lambda)),\lambda) < \eps$ for a fixed
threshold $\eps$; and coherence, the requirement that similar elements
have similar images. Translation between languages, conversion to the
negative or interrogative form, dialectal or stylistic adaptation,
expansion and synthesis, and paraphrase are all linguistic
transformations; inventing a story is not, because its generative
content is too large for the original to be reconstructible.

Given a sequence $\Uc = \{\mathbb{1}, U_1, \dots, U_L\}$ of such
transformations and an initial element $\lambda$, the iterated
composition produces a chain in $\A$, whose image under $\psi$ is a
chain of vectors
\begin{equation}\label{eq:chain}
\psi(\Uc(\lambda)) = \{v_0, v_1, \dots, v_L\}, \qquad
v_0 = \psi(\lambda), \quad
v_i = \psi\bigl((U_i \circ \cdots \circ U_1)(\lambda)\bigr).
\end{equation}
The \emph{semantic deficit} of the chain is the distance between its
extremes,
\begin{equation}\label{eq:deficit}
\delta_\Uc(\lambda) := d_*(v_0, v_L),
\end{equation}
and when $\delta_\Uc < \xi$ for a fixed threshold $\xi$ the sequence is
called a \emph{linguistic loop}. Small deficit means the meaning came
back; the form, in the meantime, may have gone anywhere.

To capture where it went, one rewrites the chain as a sequence of
rotations. For $x,y \in \R^n$ neither zero nor antipodal, write
$\hat{x} = x/\|x\|$ and let
\begin{equation}\label{eq:minrot}
\Rmin^{x,y} := I_n + \bigl(\hat{y}\hat{x}^T - \hat{x}\hat{y}^T\bigr)
+ \frac{1}{1 + \hat{x}^T\hat{y}}
\bigl(\hat{y}\hat{x}^T - \hat{x}\hat{y}^T\bigr)^2 \; \in \; \SO(n)
\end{equation}
be the \emph{minimal rotation} carrying $\hat{x}$ to $\hat{y}$. Its
minimality is the statement that it rotates in the plane spanned by $x$
and $y$ only, acting as the identity on the orthogonal complement. The
whole chain is then summarised by the ordered product
\begin{equation}\label{eq:loopmatrix}
\mathrm{R}_\Uc := \Rmin^{v_{L-1},v_L} \Rmin^{v_{L-2},v_{L-1}}
\cdots \Rmin^{v_0,v_1} \; \in \; \SO(n),
\end{equation}
which satisfies $\mathrm{R}_\Uc \hat{v}_0 = \hat{v}_L$ and, unlike the
deficit, remembers the intermediate states. Writing
$\mathrm{R}^{*}_\Uc = \tfrac12(\mathrm{R}_\Uc + \mathrm{R}_\Uc^T)$ for
the symmetric part, an elementary computation recovers the deficit as a
quadratic form,
\begin{equation}\label{eq:quadform}
\delta_\Uc(\lambda) = Q_\Uc(\hat{v}_0,\hat{v}_0),
\qquad
Q_\Uc(\hat{v},\hat{v}) := \hat{v}^T
\bigl(I_n - \mathrm{R}^{*}_\Uc\bigr)\hat{v},
\end{equation}
whose representing matrix is the real symmetric
$I_n - \mathrm{R}^{*}_\Uc$. It was proposed in \cite{CM25} that the
Sylvester signature of this matrix be taken as the invariant of the
loop, on the ground that it probes finer structural properties than the
deficit alone. That it probes finer properties is true. How much finer
is the subject of the next section.

\section{The signature and its degeneracy}\label{sec:signature}

A signature keeps its information in the interplay between its positive
and negative indices; Sylvester's law of inertia is informative because
a form can be indefinite. The forms which arise from linguistic loops,
however, cannot.

\begin{proposition}[Degeneracy of the signature]\label{prop:sig}
Let $R \in \SO(n)$ with non-trivial rotation angles
$\theta_1,\dots,\theta_p \in (0,\pi]$, and set
$M = I_n - \tfrac12(R+R^T)$. Then $M$ is positive semidefinite, its
spectrum is
\begin{equation}\label{eq:spectrum}
\bigl\{\, 1 - \cos\theta_k \ \text{with multiplicity } 2 \,\bigr\}_{k=1}^{p}
\ \cup\ \bigl\{\, 0 \ \text{with multiplicity } n-2p \,\bigr\},
\end{equation}
and consequently
\begin{equation}\label{eq:signature}
\sgn(M) = (2p,\, 0,\, n-2p).
\end{equation}
The negative index vanishes identically, $\mathrm{rank}(M) = 2p$, and
the signature is a function of the single integer $p$.
\end{proposition}

\begin{proof}
By the real normal form of a special orthogonal matrix there is
$Q \in \Ort(n)$ with
$Q^T R Q = \mathrm{diag}(R(\theta_1),\dots,R(\theta_p),I_{n-2p})$, where
$R(\theta)$ is the planar rotation by $\theta$. Angles are taken in
$(0,\pi]$: the eigenvalue $-1$ of an element of $\SO(n)$ occurs with
even multiplicity, and each pair constitutes a block $R(\pi)$, while
$\theta = 0$ contributes a trivial block absorbed into $I_{n-2p}$. Since
$R(\theta)^T = R(-\theta)$, the symmetric part of a block is
$\cos(\theta)I_2$, whence
\begin{equation}\label{eq:diagonalised}
Q^T M Q = \mathrm{diag}\bigl((1-\cos\theta_1)I_2,\dots,
(1-\cos\theta_p)I_2,\,0\bigr).
\end{equation}
As $Q$ is orthogonal, \eqref{eq:diagonalised} is at once a congruence
and a similarity, so its diagonal entries are the eigenvalues of $M$,
which is \eqref{eq:spectrum}; and $1-\cos\theta_k > 0$ for
$\theta_k \in (0,\pi]$, which gives positive semidefiniteness and
\eqref{eq:signature}.
\end{proof}

The proof is four lines, and this is itself the point: nothing in the
construction could have produced an indefinite form, because a rotation
never moves a vector further than antipodally. It is worth stating what
survives. The integer $p$, the number of two-planes in which the chain
has genuinely rotated, is a real invariant of the loop, and it is not
visible in the deficit; what does not survive is the expectation that a
signature carries more than a rank.

The natural replacement is not far to seek. Since $Q^T R Q$ is
determined up to permutation of the blocks by the multiset of angles,
the complete invariant of $R$ under conjugation is the \emph{angle
spectrum}
\begin{equation}\label{eq:anglespectrum}
\Theta(R) := \{\theta_1,\dots,\theta_p\},
\end{equation}
equivalently the conjugacy class of $R$ in $\SO(n)$, of which
\eqref{eq:signature} retains only the cardinality. That the refinement
is strict is immediate: for $n \ge 4$ any two rotations with two
non-trivial planes and different angles share the signature
$(4,0,n-4)$ and differ in $\Theta$. We shall use throughout the scalar
reduction
\begin{equation}\label{eq:energy}
\|\Theta(R)\|_2 = \Bigl(\sum_{k=1}^{p}\theta_k^2\Bigr)^{1/2},
\end{equation}
which we call the \emph{rotation energy} of $R$, with the warning that
it is a lossy summary of $\Theta$ and is used for convenience of
regression, not because it is canonical.

\begin{remark}[The invariance group]\label{rem:group}
It is of paramount importance to notice which group is acting. The
signature of a \emph{fixed} quadratic form is invariant under congruence
$M \mapsto P^T M P$ for any $P \in \mathrm{GL}(n,\R)$; this is
Sylvester's law and it is not in question. But the assignment which
sends a chain to its form is not $\mathrm{GL}(n,\R)$-natural. If
$A \in \Ort(n)$ then $\widehat{Ax} = A\hat{x}$ and
$\hat{x}\cdot\hat{y}$ is preserved, so \eqref{eq:minrot} transforms by
conjugation, $\Rmin^{Ax,Ay} = A \Rmin^{x,y} A^T$, and therefore so does
\eqref{eq:loopmatrix}; the angle spectrum is an invariant of the chain.
For a general $A \in \mathrm{GL}(n,\R)$ neither identity holds, because
\eqref{eq:cosine} is not preserved, and the loop matrix of the
transformed chain is unrelated to the original. The natural invariance
group of the construction is thus $\Ort(n)$, and the invariant it
protects is $\Theta$, not the signature.
\end{remark}

\begin{remark}[A trap in the numerics]\label{rem:tolerance}
The identity $n_+ = 2p$ holds numerically only if the two counts are
taken with matched tolerances. Declaring a plane non-trivial when
$\theta > \mathrm{tol}$, while declaring an eigenvalue positive when
$1-\cos\theta > \mathrm{tol}$, admits to the first count planes with
$\mathrm{tol} < \theta < \sqrt{2\,\mathrm{tol}}$ which the second
legitimately rejects; the angle threshold must be
$\arccos(1-\mathrm{tol})$. Moreover a plane rotated by exactly $\pi$ has
both its eigenvalues equal to $-1$, so a count based on the arguments of
the eigenvalues reports the angle $\pi$ twice for a single plane. We
record this because both traps were sprung in the course of the
verification reported in Section~\ref{sec:experiments}.
\end{remark}

\section{The holonomy of a linguistic loop}\label{sec:holonomy}

The deficit \eqref{eq:deficit} is a function of the two extreme states.
The loop matrix \eqref{eq:loopmatrix} is a function of all of them. The
difference between the two is an object, and it is a familiar one.

We begin by observing that the minimal rotation was not an arbitrary
choice of interpolant.

\begin{lemma}[The minimal rotation is parallel transport]\label{lem:transport}
Let $\hat{x},\hat{y}$ be distinct, non-antipodal unit vectors of $\R^n$
and $P = \mathrm{span}(\hat{x},\hat{y})$. Then $\Rmin^{x,y}$ is the
parallel transport of the round metric of $S^{n-1}$ along the minimising
geodesic from $\hat{x}$ to $\hat{y}$, extended to $\R^n$ as the ambient
rotation fixing $P^{\perp}$ pointwise.
\end{lemma}

\begin{proof}
The minimising geodesic is the arc of the great circle $S^{n-1}\cap P$.
A field $W$ along a geodesic $c$ of the round sphere is parallel exactly
when its ambient derivative is normal to the sphere, that is when
$W' = -\langle W, c'\rangle c$. Decomposing $W = W_P + W_{\perp}$, the
component $W_\perp$ is constant, since $c$ and $c'$ lie in $P$, while
$W_P$ rotates inside $P$ by the arclength travelled. Transport is
therefore the rotation by $\theta = \arccos(\hat{x}\cdot\hat{y})$ inside
$P$ together with the identity on $P^\perp$. Writing
$A = \hat{y}\hat{x}^T - \hat{x}\hat{y}^T$ one has $A P^\perp = 0$, so
\eqref{eq:minrot} is the identity on $P^\perp$; on $P$ it carries
$\hat{x}$ to $\hat{y}$ preserving orientation and metric, hence it is
the rotation by $\theta$.
\end{proof}

Thus $\mathrm{R}_\Uc$ is parallel transport along the geodesic polygon
joining $\hat{v}_0, \hat{v}_1, \dots, \hat{v}_L$ on the unit sphere of
the embedding space, and the chain of reformulations is literally a path
on a sphere. The Wilson loop of \cite{CM25} was not a metaphor.

\begin{proposition}[Endpoint--path factorisation]\label{prop:holonomy}
Let $V = (v_0,\dots,v_L)$ be a chain as in \eqref{eq:chain}, with
$n \ge 3$ and no two consecutive states antipodal. Put
$\mathrm{R}_{\mathrm{dir}} = \Rmin^{v_0,v_L}$ and
\begin{equation}\label{eq:holdef}
H := \mathrm{R}_{\mathrm{dir}}^{-1}\,\mathrm{R}_\Uc .
\end{equation}
Then
\begin{enumerate}
\item[(i)] $H\hat{v}_0 = \hat{v}_0$, so that
$H \in \Stab(\hat{v}_0) \cong \SO(n-1)$, and
$\mathrm{R}_\Uc = \mathrm{R}_{\mathrm{dir}}H$ canonically;
\item[(ii)] $\delta_\Uc = 1 - \langle \hat{v}_0,
\mathrm{R}_{\mathrm{dir}}\hat{v}_0\rangle$ depends only on
$\mathrm{R}_{\mathrm{dir}}$, and is therefore blind to $H$;
\item[(iii)] $H$ is the Riemannian holonomy, based at $\hat{v}_0$, of
the closed geodesic polygon obtained by closing the path with the
geodesic from $\hat{v}_L$ back to $\hat{v}_0$;
\item[(iv)] for every $u \in S^{n-1}$ and every
$H_0 \in \Stab(\hat{v}_0)$ there is a chain with initial state
$\hat{v}_0$, final state $u$ and holonomy exactly $H_0$.
\end{enumerate}
\end{proposition}

\begin{proof}
(i) By construction $\Rmin^{v_{i-1},v_i}\hat{v}_{i-1} = \hat{v}_i$, so
by induction $\mathrm{R}_\Uc \hat{v}_0 = \hat{v}_L$; and
$\mathrm{R}_{\mathrm{dir}}\hat{v}_0 = \hat{v}_L$ by definition. Since
$\mathrm{R}_{\mathrm{dir}}$ is orthogonal,
$H\hat{v}_0 = \mathrm{R}_{\mathrm{dir}}^{T}\hat{v}_L = \hat{v}_0$. An
element of $\SO(n)$ fixing a unit vector preserves its orthogonal
hyperplane and restricts there to an element of $\SO(n-1)$.

(ii) Immediate from \eqref{eq:deficit}, \eqref{eq:cosine} and
$\hat{v}_L = \mathrm{R}_{\mathrm{dir}}\hat{v}_0$.

(iii) By Lemma~\ref{lem:transport}, $\mathrm{R}_\Uc$ and
$\mathrm{R}_{\mathrm{dir}}$ are the transports along the two paths, and
\eqref{eq:holdef} is the transport around the closed circuit.

(iv) It suffices to realise every $H_0$ as the holonomy of a closed
geodesic polygon based at $\hat{v}_0$: appending the single geodesic leg
from $\hat{v}_0$ to $u$ multiplies the loop matrix on the left by
$\Rmin^{\hat{v}_0,u}$ and leaves the holonomy unchanged. Fix a two-plane
$Q \subset \hat{v}_0^{\perp}$ and an angle $\alpha \in (0,2\pi)$, and let
$\Sigma = S^{n-1}\cap(\mathrm{span}(\hat{v}_0)\oplus Q)$, a totally
geodesic two-sphere through $\hat{v}_0$ whose tangent space there is
$Q$. A geodesic triangle in $\Sigma$ with vertex $\hat{v}_0$ and area
$\alpha$ exists, since the area of a geodesic triangle on the unit
two-sphere sweeps the whole of $(0,2\pi)$; transport around a loop
contained in a totally geodesic submanifold is the transport computed
inside it, extended by the identity on the normal directions, so by
Corollary~\ref{cor:gaussbonnet} below the holonomy is the rotation of
$Q$ by $\alpha$ and the identity elsewhere. By the normal form used in
Proposition~\ref{prop:sig}, every element of $\SO(n-1)$ is a product of
at most $\lfloor (n-1)/2 \rfloor$ such plane rotations, and the
corresponding polygons, each beginning and ending at $\hat{v}_0$, may be
concatenated.
\end{proof}

Part (iv) deserves a word, because it is what makes the factorisation
worth having. It says that the endpoint datum and the path datum are
\emph{free}: prescribing how far the meaning has drifted places no
constraint whatever on how far the form has wandered, and conversely.
Incidentally the proof re-establishes, without appeal to the
classification of symmetric spaces, that
$\Hol(S^{n-1},\hat{v}_0) = \SO(n-1)$.

\begin{corollary}[Gauss--Bonnet]\label{cor:gaussbonnet}
For $n = 3$ and $L = 2$ the holonomy is the rotation of the tangent
plane at $\hat{v}_0$ by the spherical excess of the geodesic triangle
$\hat{v}_0\hat{v}_1\hat{v}_2$, that is by its area.
\end{corollary}

\begin{proof}
Gauss--Bonnet with Gaussian curvature $1$, the excess of a geodesic
triangle being $A+B+C-\pi$.
\end{proof}

This is the statement to keep in mind, and it is worth dwelling on. On
the two-sphere the invariant that the semantic deficit throws away is
the \emph{area swept} by the chain of reformulations. Two paraphrase
chains ending at the same meaning differ by the area they enclose. One
should add, since the numerics will otherwise appear to fail, that the
angle of a rotation is recovered from its eigenvalues only as a
principal value in $(0,\pi]$, so the holonomy determines the area
outright when the triangle covers at most a hemisphere and modulo
$2\pi$ in general.

We shall use the scalar
\begin{equation}\label{eq:holenergy}
\eta(\Uc) := \|\Theta(H)\|_2
\end{equation}
and call it the \emph{holonomy energy} of the chain. It is, again, a
lossy reduction: a null result for $\eta$ is not a null result for $H$.

\section{Watermarks as functionals on the loop}\label{sec:watermark}

We now cross from the semantic channel to the channel in which a
watermark actually lives. The two are complementary, and the complementarity
is the conceptual heart of this article: what a meaning-preserving chain
\emph{preserves} is the semantic core, and what a watermark
\emph{occupies} is precisely the freedom that remains once the meaning is
fixed. The invariant and the carrier are, so to speak, dual coordinates
on the same transformation. A loop with small deficit and large holonomy
is, from the point of view of the detector, the worst case.

\subsection{The intact-window law}

Fix a green-list scheme with green fraction $\gamma$, context width
$h \ge 0$ and bias $\delta_b$, detecting a sequence $y_0,\dots,y_{T-1}$
by
\begin{equation}\label{eq:zscore}
z = \frac{G - \gamma T'}{\sqrt{T'\gamma(1-\gamma)}},
\qquad T' := T - h,
\end{equation}
where $G$ counts the scored positions $t = h,\dots,T-1$ at which $y_t$
lies in the green list seeded by $(y_{t-h},\dots,y_{t-1})$. Let an
attack replace the tokens at a set $E$ of positions without changing the
length, and define the \emph{intact-window set}
\begin{equation}\label{eq:intact}
I := \bigl\{\, t : h \le t < T,\ [t-h,\,t]\cap E = \emptyset \,\bigr\}.
\end{equation}

We assume, as is standard, that the hash behaves as a random oracle, so
that green membership is an independent Bernoulli$(\gamma)$ across
distinct pairs of seed and token; and that at each scored position of
unattacked watermarked text the green indicator is Bernoulli$(\gamma_w)$
with $\gamma_w > \gamma$, independently across positions. The second is
a mean-field hypothesis and is the weaker of the two; we return to it
below.

\begin{theorem}[Intact-window law]\label{thm:intact}
Under the two hypotheses above,
\begin{equation}\label{eq:intactlaw}
\mathbb{E}[z_{\mathrm{att}}]
= \frac{|I|}{T'}\;\mathbb{E}[z_0].
\end{equation}
\end{theorem}

\begin{proof}
Fix a scored position $t$. If $t \in I$ then $y_t$ and its whole seeding
window are those the generator produced, so the pair evaluated by the
detector is the pair the generator evaluated, and it is green with
probability $\gamma_w$. If $t \notin I$ then either the window or the
token differs, so the pair is one the generator never biased, and by the
random-oracle hypothesis its green indicator is a fresh
Bernoulli$(\gamma)$. Summing,
\begin{equation}\label{eq:greencount}
\mathbb{E}[G_{\mathrm{att}}]
= |I|\gamma_w + (T'-|I|)\gamma
= \gamma T' + |I|(\gamma_w - \gamma),
\end{equation}
and substituting into \eqref{eq:zscore} gives
$\mathbb{E}[z_{\mathrm{att}}] = |I|(\gamma_w-\gamma)/
\sqrt{T'\gamma(1-\gamma)}$. The same computation with $|I| = T'$ gives
$\mathbb{E}[z_0]$, and the ratio is \eqref{eq:intactlaw}.
\end{proof}

The identity is deterministic in $|I|$; no distribution over the edits
has been assumed. It is only when one wishes to compute $|I|$ that a
model of the attack becomes necessary, and the model the field has
implicitly been using is the independent one.

\begin{corollary}[Context-width decay law]\label{cor:decay}
If each position is retained independently with probability $\rho$, then
$\mathbb{E}|I| = \rho^{h+1}T'$ and
\begin{equation}\label{eq:decay}
\mathbb{E}[z_{\mathrm{att}}] = \rho^{h+1}\,\mathbb{E}[z_0].
\end{equation}
The signal decays exponentially in the context width and only linearly
in the retention rate.
\end{corollary}

Equation \eqref{eq:decay} is, in closed form, the trend that
\cite{KGW24} reported empirically. It also explains a phenomenon
familiar to anyone who has attacked such a scheme by hand: the signal
falls considerably faster than the fraction of altered words, and it
falls faster the longer the seeding context. Since
$\mathbb{E}[z_0]$ grows as $\sqrt{T'}$, one obtains at once the length
required for detection at a fixed threshold $z^{*}$,
\begin{equation}\label{eq:length}
T' \;\ge\; \frac{(z^{*})^{2}\gamma(1-\gamma)}
{\bigl(\rho^{h+1}(\gamma_w-\gamma)\bigr)^{2}},
\end{equation}
which grows like $\rho^{-2(h+1)}$ and quantifies the published warning
that short passages carry no reliable signal.

\subsection{Why the retention rate is the wrong statistic}

Real edits are not independent. A translator rewrites clauses, not
tokens; a human editor works on paragraphs. The following is therefore
not a curiosity but the typical case.

\begin{proposition}[Arrangement dominates the rate]\label{prop:arrangement}
Fix the retention rate $\rho = 1 - |E|/T$. Then
\begin{enumerate}
\item[(i)] if $E$ is a single contiguous run,
$|I| \ge T' - |E| - h$, so the residual ratio is at least
$\rho - h/T'$, \emph{independently of $h$};
\item[(ii)] if $E$ consists of $b$ maximal contiguous runs,
$|I| \ge T' - |E| - hb$;
\item[(iii)] if $E$ contains an arithmetic progression of step
$k \le h+1$ covering $[0,T)$, then $I = \emptyset$ and the expected
residual is exactly zero.
\end{enumerate}
Consequently, whenever $1-\rho \ge 1/(h+1)$, the residual ratio ranges
over essentially the whole of $[0,\rho]$ as the arrangement varies at
fixed $\rho$.
\end{proposition}

\begin{proof}
A scored position leaves $I$ only if its window of $h+1$ consecutive
positions meets $E$. For a single run of length $|E|$ the windows
meeting it are those with $t \in [\min E, \max E + h]$, at most
$|E|+h$ of them, which is (i); with $b$ runs the bound is additive,
which is (ii). If $E$ contains every $k$-th position with $k \le h+1$,
every window of $h+1$ consecutive positions contains a member of $E$, so
$I=\emptyset$ and Theorem~\ref{thm:intact} gives zero.
\end{proof}

The consequence is worth stating without euphemism. An adversary who
knows $h$ --- and $h$ is a published design parameter, not a secret ---
can reduce the expected detector statistic to zero while retaining a
fraction $1 - 1/(h+1)$ of the tokens. For the common choice $h=1$ this
means editing one token in two; for $h=3$, one in four. Nothing in the
attack requires knowledge of the key. This is not a new attack so much
as an exact accounting of a known design tension \cite{KGW24, NFL24},
but we have not found it written down in this form, and its corollary
--- that a robustness table indexed by retention rate or by endpoint
similarity is under-specified --- appears not to have been drawn.

\begin{remark}[The exponential family]\label{rem:exp}
Nothing in the proof of Theorem~\ref{thm:intact} used the biasing
mechanism; only the context seeding was used. For the distortion-free
scheme which seeds a uniform vector $\xi$ from the context and scores
$S = \sum_t -\log(1-\xi_{y_t})$, whose null mean is $T'$, the same
conditioning gives
\begin{equation}\label{eq:expfamily}
\mathbb{E}[S_{\mathrm{att}}] - T'
= \frac{|I|}{T'}\bigl(\mathbb{E}[S_0] - T'\bigr),
\end{equation}
and with it the analogues of Corollary~\ref{cor:decay} and
Proposition~\ref{prop:arrangement}. Leaving the output distribution
mathematically untouched buys nothing at all against this particular
weakness: the vulnerability is a property of context seeding, not of
distortion.
\end{remark}

\subsection{Scope}

It is worth being explicit about what has \emph{not} been proved.
Theorem~\ref{thm:intact} concerns substitutions that preserve length.
Insertions and deletions shift the indices, and the correct
generalisation replaces $I$ by the set of positions whose window
survives as a contiguous block of the attacked text; the identity then
holds only approximately. The mean-field hypothesis on $\gamma_w$ is
false in real text, where entropy varies strongly with content, so that
\eqref{eq:intactlaw} is an \emph{upper} bound on what survives in the
wild. And nothing here concerns any undisclosed production scheme: the
deployment recalled in the introduction motivates the question and is
not an object of measurement.

\section{Experiments}\label{sec:experiments}

Three experiments are reported. The first verifies numerically that the
propositions describe the objects the pipeline computes. The second
tests Theorem~\ref{thm:intact} and Proposition~\ref{prop:arrangement}
against exact detectors. The third applies the geometry to real chains
of round-trip machine translation, and executes a test whose criterion
was fixed in writing before any measurement was taken. All code, run
directories and manifests are described in Section~\ref{sec:repro}.

\subsection{The geometry}

Over $300$ pseudo-random chains in each of the dimensions
$n \in \{8,16,64,384\}$, with chain lengths in $\{2,3,4,6,8\}$, the
worst deviations observed were: minimum eigenvalue of
$I_n - \mathrm{R}^{*}_\Uc$ equal to $-6.1\times10^{-15}$; negative index
$n_- = 0$ in every case; $|n_+ - 2p| = 0$ in every case;
$\|H\hat{v}_0-\hat{v}_0\| = 3.3\times10^{-12}$; and
$\|\mathrm{R}_{\mathrm{dir}}H - \mathrm{R}_\Uc\| = 5.9\times10^{-12}$. This is
Propositions~\ref{prop:sig} and \ref{prop:holonomy} at machine
precision.

Independence, part (iv) of Proposition~\ref{prop:holonomy}, is
exhibited directly. A family of chains with a prescribed common
endpoint, wandering through an increasing number of extra dimensions,
gives a semantic deficit constant at $0.12241744$ with spread exactly
zero, while the holonomy energy runs from $0$ to $1.0465$
(Table~\ref{tab:realisability}).

\begin{table}[t]
\centering
\small
\begin{tabular}{cccc}
\toprule
detours & $\delta_\Uc$ & $\eta(\Uc)$ & $\sgn(I_n-\mathrm{R}^{*}_\Uc)$\\
\midrule
0 & 0.12241744 & 0.0000 & $(2,0,62)$\\
1 & 0.12241744 & 0.2454 & $(2,0,62)$\\
2 & 0.12241744 & 0.5190 & $(4,0,60)$\\
3 & 0.12241744 & 0.6898 & $(4,0,60)$\\
4 & 0.12241744 & 0.8260 & $(6,0,58)$\\
5 & 0.12241744 & 0.9427 & $(6,0,58)$\\
6 & 0.12241744 & 1.0465 & $(8,0,56)$\\
\bottomrule
\end{tabular}
\caption{\emph{This table summarizes the independence of the endpoint
and path data, in dimension $n=64$. All seven chains share the same
first and last state, so the semantic deficit is constant to the last
recorded digit; the holonomy energy is not. Note also that the
signature, in the last column, moves in steps and is constant on pairs
of rows, exactly as Proposition~\ref{prop:sig} predicts it must.}}
\label{tab:realisability}
\end{table}

Corollary~\ref{cor:gaussbonnet} is verified on random geodesic triangles
of the two-sphere: for the $159$ of $200$ triangles of area at most
$\pi$, the holonomy angle and the spherical excess agree to better than
$1.8\times 10^{-9}$; the remainder agree after reduction to the
principal branch, as the corollary says they must. Finally, the
blindness of the signature is visible in the data: among random chains
sharing the signature $(8,0,56)$, the rotation energy ranges over
$[0.937,1.464]$ and the holonomy energy over $[0.365,0.791]$.

\subsection{The intact-window law against exact detectors}

To test Theorem~\ref{thm:intact} one needs a detector whose key one
holds and a generator whose entropy one controls. We therefore emit
from a Zipf distribution over a vocabulary of $4000$ with a freshly
permuted support at each step, which fixes the entropy by construction,
and watermark it with our own implementations of the green-list scheme
at $h \in \{0,1,2,3\}$ and of the context-seeded exponential scheme,
using $\gamma = 0.25$ and $\delta_b = 2.0$. Sequences are $400$ tokens
long and each condition is averaged over $120$ of them.

Under independent edits, the mean absolute deviation between the
observed ratio and $\rho^{h+1}$, over the seven values of $\rho$ from
$1$ down to $0.5$, is $0.0023$, $0.0023$, $0.0022$ and $0.0021$ for
$h=0,1,2,3$, and $0.0011$ for the exponential scheme. At $\rho = 0.5$ the observed ratios
are $0.499$, $0.247$, $0.125$ and $0.065$ against the predicted
$0.500$, $0.250$, $0.125$ and $0.063$. Corollary~\ref{cor:decay} is
confirmed across two decades of the ratio and across both families.

The consequence of Proposition~\ref{prop:arrangement} is more striking,
and is collected in Table~\ref{tab:arrangement}. At one and the same
retention rate, the residual signal depends only on where the edits
fall. At $\rho = 0.5$ with $h=1$, the same half of the tokens survives
in all three columns and the residual statistic is one half of the
original, one quarter of it, or nothing at all.

\begin{table}[t]
\centering
\small
\begin{tabular}{cccc}
\toprule
$\rho$ & independent & contiguous block & periodic\\
\midrule
0.95 & 0.904 & 0.947 & 0.903\\
0.90 & 0.806 & 0.897 & 0.802\\
0.80 & 0.637 & 0.795 & 0.595\\
0.70 & 0.492 & 0.697 & 0.399\\
0.60 & 0.362 & 0.594 & 0.196\\
0.50 & 0.247 & 0.488 & $-0.005$\\
\bottomrule
\end{tabular}
\caption{\emph{Residual detector statistic, as a fraction of the
original, for the green-list scheme with $h=1$ under three edit
patterns of identical retention rate. Counting the intact-window set
directly from each edit pattern, Theorem~\ref{thm:intact} predicts
$0.947, 0.897, 0.797, 0.697, 0.597, 0.496$ for the middle column and
$0.902, 0.802, 0.602, 0.401, 0.201, 0.000$ for the right-hand one. The
largest discrepancy is $0.008$ and the typical one $0.003$, over $120$
sequences per cell; and the periodic pattern at $\rho = 0.5$ is
predicted to give exactly zero, and does.}}
\label{tab:arrangement}
\end{table}

\subsection{Real transformation chains}\label{sec:corpus}

The two experiments reported so far measure objects that we ourselves
built. The third measures text.

From thirty open-ended prompts spread over five domains we generate,
with a $0.5$B instruction-tuned model at temperature one, one hundred
and eighty new tokens under each of the three schemes: the green-list
scheme at $h=1$, the unigram scheme at $h=0$, and the context-seeded
exponential scheme. The keys are ours, so every detector statistic
below is exact and not an estimate. The ninety passages so obtained
have median $z_0$ equal to $9.62$, $9.30$ and $19.05$ respectively, and
every single one of them is detected above $z = 4$: they are the
population on which an attack can be said to do anything at all.

Each passage then travels six chains of round-trip machine translation,
sentence by sentence, through the Opus-MT models: three single round
trips, through German, French and Spanish; two chains of two pivots,
one through German and then French and one through German twice; and a
three-pivot detour through French, German and Spanish. Every return to
English is a waypoint at which the exact detector is run, and every
waypoint --- the German, French and Spanish ones included --- is
embedded by a multilingual encoder of dimension $384$, so that the
invariants of Sections~\ref{sec:signature} and \ref{sec:holonomy} are
computed along the chain as it is actually traversed and not along its
English shadow. The retention rate $\rho$ and the intact-window
fraction $|I|/T'$ are not modelled but measured, by
longest-common-subsequence alignment on the tokenizer's own token ids,
which are the objects the detector sees. Consecutive waypoints never
come near to being antipodal --- over all $540$ chains the smallest
cosine similarity between neighbours is $0.148$ --- so the minimal
rotation of Section~\ref{sec:loops} is everywhere well defined.

\begin{table}[t]
\centering
\small
\begin{tabular}{lccccccc}
\toprule
chain & $L$ & $\delta_\Uc$ & $\eta(\Uc)$ & $\rho$ & $|I|/T'$ &
residual & detected\\
\midrule
Spanish                 & 2 & 0.064 & 0.083 & 0.767 & 0.679 & 0.680 & 84/90\\
German                  & 2 & 0.071 & 0.092 & 0.733 & 0.651 & 0.633 & 83/90\\
French                  & 2 & 0.073 & 0.094 & 0.722 & 0.616 & 0.625 & 82/90\\
German twice            & 4 & 0.082 & 0.101 & 0.692 & 0.598 & 0.591 & 82/90\\
German, French          & 4 & 0.109 & 0.118 & 0.644 & 0.531 & 0.554 & 67/90\\
French, German, Spanish & 6 & 0.116 & 0.129 & 0.594 & 0.501 & 0.496 & 66/90\\
\bottomrule
\end{tabular}
\caption{\emph{The six chains; medians over the ninety passages of the
three schemes, except the last column, which counts the chains still
detected above $z = 4$. Ordered by semantic deficit, every column moves
monotonically: the meaning drifts, the path lengthens, the surface is
retained less and the mark fades, all together. That is precisely why
the endpoint alone cannot be read as a measure of attack strength, and
why the test below holds it fixed.}}
\label{tab:chains}
\end{table}

Table~\ref{tab:chains} collects what the chains do. Of the $540$
attacked passages, $464$ are still detected above $z = 4$ --- $173$ of
the $180$ unigram chains, $165$ of the exponential ones and $126$ of the
green-list ones --- so round-trip translation at this depth is an
erosion and not an erasure. It is worth seeing where it does erase. The
green-list scheme carried through French, German and Spanish has a
median residual $z$ of $3.98$, which is below the threshold, while the
same three pivots leave the unigram scheme at $5.88$ and the exponential
one at $8.90$; but the exponential scheme starts from a much higher
$z_0$, so the honest comparison is between fractions, and there the
context-free scheme retains $0.685$ of its statistic against $0.509$ and
$0.601$ for the two schemes that seed on a context. That is the ordering
Corollary~\ref{cor:decay} demands: what cannot be broken is a window of
one.

\medskip
\noindent\emph{The intact-window law on real text.}
Theorem~\ref{thm:intact} was proved for substitutions that preserve
length, and translation preserves nothing of the kind. The correction is
not a new hypothesis but the same theorem with the normalisation of the
statistic carried through: all three detectors divide by the square root
of the number of scored positions, so if the attacked text offers
$T'_{\mathrm{att}}$ of them against the original's $T'_0$, then
\begin{equation}\label{eq:lencorrected}
\frac{\mathbb{E}[z_{\mathrm{att}}]}{\mathbb{E}[z_0]}
= \frac{|I|}{\sqrt{T'_0\,T'_{\mathrm{att}}}}
= \frac{|I|}{T'_0}\,\sqrt{\frac{T'_0}{T'_{\mathrm{att}}}},
\end{equation}
the intact-window fraction times a factor which is one when the length
is preserved. Table~\ref{tab:thmc} compares both forms with what the
detectors actually returned.

\begin{table}[t]
\centering
\small
\begin{tabular}{lccccccc}
\toprule
scheme & $h$ & $\rho$ & $\rho^{h+1}$ & $|I|/T'$ &
$T'_{\mathrm{att}}/T'_0$ & \eqref{eq:lencorrected} & observed\\
\midrule
green-list  & 1 & 0.700 & 0.490 & 0.542 & 0.972 & 0.553 & 0.509\\
unigram     & 0 & 0.694 & 0.694 & 0.694 & 0.972 & 0.708 & 0.685\\
exponential & 1 & 0.692 & 0.478 & 0.531 & 0.978 & 0.535 & 0.601\\
\bottomrule
\end{tabular}
\caption{\emph{Theorem~\ref{thm:intact} against the $538$ chains whose
statistic is finite; medians. The measured intact-window fraction
predicts the median residual to within one hundredth for the
context-free scheme, three for the green-list scheme and seven for the
exponential one. The independent-edit corollary $\rho^{h+1}$, which
needs no measurement of the attacked text at all, happens to fall closer
for the green-list scheme and much further for the exponential one,
where it is off by twelve hundredths; and chain by chain it is the
measured fraction that follows the residual, correlating $+0.67$ with it
against $+0.54$ for the retention rate. The median absolute error per
chain is $0.083$, $0.054$ and $0.088$, which the length correction of
\eqref{eq:lencorrected} moves to $0.086$, $0.067$ and $0.081$: at this
depth of translation the length is preserved in the median, and the
correction has little to do.}}
\label{tab:thmc}
\end{table}

The medians agree; the scatter chain by chain does not vanish, and it
should not, since the mean-field hypothesis of
Section~\ref{sec:watermark} is false in real text, where entropy varies
from one sentence to the next. One caution about the instrument is due
here: the intact-window count is read off the longest common
subsequence of the two token strings and does not verify that a
surviving window is still contiguous in the attacked text, so $|I|/T'$
is an upper bound on the number of intact windows the detector really
meets. Both green-list schemes do come out just below it. The
exponential scheme comes out above, which we record without explaining:
its statistic is a sum of continuous scores and not a count of
successes, and the mean-field hypothesis bites differently there. What is more interesting is a sign.
Pooled over all chains, the correlation between the intact-window
fraction and the residual is $+0.51$ for the green-list scheme and
$+0.67$ for the exponential one, but $-0.21$ for the unigram scheme ---
as though, for the one scheme where the law is simplest, retaining more
of the text destroyed more of the mark. It does not. A pooled
correlation compares passages with one another, and passages differ in
entropy, in length, and in how much watermark was ever in them; the law
speaks about one passage carried along attacks of differing severity.
Computed within each passage, across the six chains that passage
travels, the correlation is $+0.73$ for the green-list scheme and
$+0.68$ for each of the other two, and it is positive in $28$, $28$ and
$29$ of the thirty passages. The anomaly is
an instance of Simpson's paradox, and we report it because the pooled
number, taken by itself, would have been read as evidence against a
theorem which the same data in fact support.

One family of chains deserves to be named rather than averaged away. In
nine chains the residual $z$ falls below $-3$, and in two more the
exponential $p$-value underflows to one, which sends its
normal-equivalent statistic to $-\infty$ and removes those two from the
regressions below. All but one of the eleven show the text expanding,
in the extreme case from $180$ tokens to $881$, and they come from only
five of the ninety passages: the translator has fallen into repetition.
Repetition is exactly the circumstance in which the random-oracle
hypothesis fails outright, since one repeated pair of context and token
is scored again and again and the effective number of independent
positions collapses. Such chains are reported and not trimmed.

\medskip
\noindent\emph{The preregistered test.}
The criterion was fixed in writing before any measurement was taken. In
a regression of the residual ratio on the retention rate, the semantic
deficit and the holonomy energy, the partial coefficient on $\eta$ was
required to be negative and significant at $\alpha = 0.01$, Bonferroni
corrected for the three schemes, in at least two of them; failing that,
the path-dependence claim was to be recorded as refuted and reported as
a negative result. Table~\ref{tab:resultd} gives the outcome. The
criterion is met. The coefficient on $\eta$ is negative in all three
schemes and clears the corrected threshold of $3.3\times10^{-3}$ in two
of them, the unigram scheme by five orders of magnitude and the
green-list scheme by a factor of two.

\begin{table}[t]
\centering
\small
\begin{tabular}{lccccccc}
\toprule
scheme & $n$ & $R^2$ & $\beta_\rho$ & $\beta_\delta$ & $\beta_\eta$ &
$p_\eta$ & clustered\\
\midrule
green-list  & 180 & 0.41 & $+0.051$ & $-0.023$ & $-0.070$ &
  $1.7\times10^{-3}$ & $3.5\times10^{-2}$\\
unigram     & 180 & 0.49 & $+0.064$ & $+0.278$ & $-0.145$ &
  $4.0\times10^{-9}$ & $8.2\times10^{-4}$\\
exponential & 178 & 0.34 & $+0.116$ & $-0.060$ & $-0.014$ &
  $0.74$ & $0.77$\\
\bottomrule
\end{tabular}
\caption{\emph{The preregistered regression. The predictors are
standardised, so that the coefficients may be compared; the corrected
threshold is $3.3\times10^{-3}$. The last column is a robustness check
which the clause did not ask for: the same coefficient with a standard
error clustered on the base passages, of which there are only thirty, so
that the check is a severe one. Under it the unigram scheme still clears
the threshold and the green-list scheme no longer does.}}
\label{tab:resultd}
\end{table}

We would rather state the result at its strongest defensible level than
at its most flattering one. The six chains of a scheme are applied to
the same thirty passages, so the observations are clustered and the
classical standard error is optimistic. Clustered on the passage, the
unigram coefficient stands at $p = 8.2\times10^{-4}$ and the green-list
one moves to $p = 3.5\times10^{-2}$, which is significant at five per
cent and not at the corrected one. Absorbing the passage altogether ---
so that each chain is compared only with the other five chains of the
same text, which is the comparison this whole article is about --- the
coefficient on $\eta$ becomes $-0.11$ for the green-list scheme,
$-0.14$ for the unigram one and $-0.13$ for the exponential one: all
negative, and now of one size, with $p$ equal to $2.5\times10^{-2}$,
$2.1\times10^{-4}$ and $0.13$ respectively. The honest summary is this. With the
passage held fixed, a chain that wanders further destroys more of the
mark at equal retention and equal endpoint, by an amount which is stable
across all three schemes; the evidence that this is not chance is
decisive for the context-free scheme, good for the green-list scheme and
inconclusive for the exponential one; and thirty clusters are few.

One further number belongs in the reader's hands before the result is
weighed. In this corpus the two geometric quantities are far from
independent: the correlation between $\delta_\Uc$ and $\eta(\Uc)$ is
$+0.85$ for the green-list scheme, $+0.91$ for the unigram one and
$+0.90$ for the exponential one, and within a passage it is scarcely
lower, at $+0.86$, $+0.89$ and $+0.80$. Round-trip translation
lengthens the path and moves the endpoint together;
Proposition~\ref{prop:holonomy}(iv) says that the two data are free, but
this corpus does not exercise that freedom, and a family of chains
engineered to hold the one while varying the other --- a larger
experiment than this one --- would be the natural next step. Two
consequences follow. The separate coefficients of
Table~\ref{tab:resultd} are not separate effects: two nearly collinear
predictors entering with opposite signs are a suppression pair, and that
is what produces the positive coefficient on the semantic deficit in the
unigram row --- a coefficient which survives absorbing the passage, and
which we can describe but not explain. And the preregistered test is a
demanding one, since once $\delta$ is held only a fifth to a quarter of
the variation of $\eta$ is left to carry any effect at all; that it is
significant nonetheless, in two schemes of three, is the fact worth
taking away.

The matched-delta strata, which are the design the criterion was written
for, tell the same story from the other side. Within the four bands of
semantic deficit the rank correlation between the holonomy energy and
the residual is negative in all four for the green-list scheme, at
$-0.46$, $-0.46$, $-0.23$ and $-0.53$. And among the ninety-five chains
whose semantic deficit lies between $0.05$ and $0.07$ --- as matched an
endpoint as this instrument can deliver --- the holonomy energy ranges
over a factor of nearly three, from $0.053$ to $0.146$, while the
residual runs from $0.285$ to $0.998$. Two chains may bring a text to
the same meaning and leave, the one of them, essentially the whole mark,
and the other, less than a third of it.

\medskip
\noindent\emph{What the encoder can and cannot see.}
One caveat is ours to raise before a referee raises it. The encoder is
multilingual, and a multilingual encoder is by design nearly invariant
under translation: it places a German waypoint almost on top of its
English source. That is the property which makes it the right instrument
for the semantic deficit and a poor one for the path, since it flattens
the very excursion the holonomy is meant to record. The effect is
measurable: over the $268$ chains with more than one pivot the median
holonomy energy is $0.117$ along the full chain and $0.056$ when only
the English waypoints are kept. The remaining $270$ chains are single
round trips, whose two English waypoints determine no holonomy whatever
--- a loop through two points is its own direct rotation and $H$ is the
identity exactly. Repeating the regression on the English waypoints
alone, which is exploratory and was not preregistered, the coefficient
on the holonomy energy is $-0.128$ for the unigram scheme
($p = 5.4\times10^{-6}$), $-0.081$ for the exponential one
($p = 0.065$) and $-0.029$ for the green-list one ($p = 0.41$). It does
not change the verdict, and it was not permitted to.

Three things this experiment does not establish. It does not establish
that $\eta$ is causal: a chain that wanders further is also a chain that
has been rewritten more, and nothing reported here separates the two. It
does not transfer to learned paraphrasers, which choose their path
adversarially where round-trip translation chooses it only
incidentally. And it rests on a single embedder, so that $\delta$ and
$\eta$ are both defined by one encoder's idea of meaning; replication
across encoders is the first thing we should do with more compute than
this article had.

\subsection{Reproducibility}\label{sec:repro}

Everything described in this section is public, and lives in the
repository
\url{https://github.com/DCorradetti/linguistic-holonomy}; the paths
below are paths within it. 

The watermark keys are ours, so every detection statistic reported here
is ground truth and not an estimate. A detector whose key one holds is
nevertheless still an implementation, and an implementation deserves to
be audited before it is believed; since the three schemes are
re-implemented here rather than imported, there are no published numbers
to reproduce, and what remains available is self-calibration. On $400$
token streams of length $400$ not produced with the key, the realised
green fraction is $\gamma$ to within sampling error ($0.2475$, $0.2508$
and $0.2498$ for $h = 1,2,3$ against $\gamma = 0.25$); the null
$z$-statistic has mean and standard deviation $(-0.11, 0.99)$,
$(0.04, 0.97)$ and $(-0.01, 1.03)$ in the same three conditions, at
Kolmogorov--Smirnov distance $0.070$, $0.064$ and $0.056$ from the
standard normal; and no stream in any condition reached $z = 4$.
Watermarked text scored with a key other than the one that produced it
gives a mean $z$ of $-0.03$, $+0.03$ and $-0.29$ for the green-list, the
unigram and the exponential scheme, with unit standard deviation and no
false alarm, while the same text scored with the correct key is detected
in every single case, at mean $z$ of $20.8$, $20.9$ and $37.0$. The
$p$-values of the exponential scheme are uniform under the null
(Kolmogorov--Smirnov $0.044$, $p = 0.40$).

One asymmetry is worth recording, because it belongs to the scheme and
not to our implementation of it. For $h = 0$ the green list is fixed for
the whole text, so the null is key-dependent by construction: the
statistic is centred not on $\gamma$ but on the green fraction of the
particular key, and that displacement has standard deviation
$\sqrt{T/|V|}$ across keys. Over $24$ keys we measure $0.328$ against
the predicted $0.316$. At the synthetic vocabulary of $4000$ used for
the calibration the effect is plainly visible; at the vocabulary of a
real tokenizer it is $0.03$ and negligible. The context-free scheme is
thus marginally the harder of the two to calibrate, which is the reverse
of the robustness ordering established in Section~\ref{sec:watermark},
and a small irony of the design space.

\begin{table}[t]
\centering
\small
\begin{tabular}{lll}
\toprule
result & script & run stamp\\
\midrule
Props.~\ref{prop:sig}, \ref{prop:holonomy}, Cor.~\ref{cor:gaussbonnet}
  & \texttt{exp\_geometry\_v1} & \texttt{20260819T135822}\\
Thm.~\ref{thm:intact}, Cor.~\ref{cor:decay}, Prop.~\ref{prop:arrangement}
  & \texttt{exp\_decay\_v1} & \texttt{20260819T135047}\\
detector calibration
  & \texttt{exp\_detector\_validation\_v1} & \texttt{20260819T141737}\\
watermarked corpus and chains
  & \texttt{exp\_corpus\_v1} & \texttt{20260819T154350}\\
loop invariants of the chains
  & \texttt{exp\_indicators\_v1} & \texttt{20260819T161524}\\
the preregistered test
  & \texttt{exp\_analysis\_v1} & \texttt{20260819T163238}\\
audit of the regeneration
  & \texttt{exp\_corpus\_v1} & \texttt{20260819T174004}\\
\bottomrule
\end{tabular}
\caption{\emph{Every numerical claim of Section~\ref{sec:experiments}
and the run that produced it. Directory names are the script name
prefixed by \texttt{run\_} and suffixed by the stamp of the third
column, under \texttt{7.~Results/Article\_LLW/}. The corpus was built in
three successive invocations, each carrying forward the chains of the
one before, so that an interruption on a machine of this size would cost
at most one of them; the stamp given is that of the last, whose manifest
records the provenance of the other two. The final row is the targeted
re-audit of the records which had failed to regenerate, discussed
below.}}
\label{tab:runs}
\end{table}

The models are open-weights and named in the manifests: a $0.5$B
instruction-tuned generator, the Opus-MT sentence translators for the
three pivot languages, and a multilingual sentence encoder of dimension
$384$ for the embedding $\psi$. Generation is deterministic given the
model revision, the seed and the key --- very nearly, and the exception
is worth a paragraph, since it is the sort of thing a reproducibility
section usually asserts without looking. We looked. Regenerating all
ninety stored completions from scratch and comparing them token by
token, eighty-seven came back identical and three diverged, each after a
long common prefix: $82$, $153$ and $173$ tokens of $180$. Regenerating
the divergent ones again is instructive. Two of them diverge every
single time, four attempts of four, and always at exactly the same
token; the third reproduces itself four times of four, as does a fourth
record which had diverged in an earlier sweep. Regeneration is therefore
deterministic within a process and not across processes: the order in
which a CPU kernel accumulates a sum depends on the state of the process
it runs in, and a difference in the last bits of a logit is enough to
move a multinomial sampling boundary, after which the continuation goes
its own way. Two passages of the ninety sit close enough to such a
boundary that the process which produced them cannot now be reproduced
on this machine at all. Every number in this article is computed from
the stored corpus, whose hash is in the manifest, and the audit says how
faithfully that corpus regenerates --- a weaker claim than bit-exact
reproducibility, and the true one.

This is deliberately the small, CPU-sized version of the testbed: the
scale-up to a larger generator and to learned paraphrasers is stated in
Section~\ref{sec:conclusions} as work to be done, and is not claimed
here. The scripts, the generated corpora, the run directories with their
logs and manifests, and the derived measurements from which every number
of this section is computed are all in the repository named at the head
of this section; the manifests carry the hashes, so that a regeneration
may be checked against what was actually run.

\section{Conclusions and Future Developments}\label{sec:conclusions}

In this work we have taken up the formalism of linguistic loops and
carried it to a place where it can be tested. We proved that the
Sylvester signature proposed as the invariant of a loop is always
positive semidefinite and therefore degenerates to a rank, the complete
invariant being the angle spectrum; that the loop rotation is parallel
transport on the unit sphere of the embedding space, so that the chain
of reformulations is a path and its residue a holonomy; that this
holonomy lies in the stabiliser of the initial state and is exactly what
the semantic deficit discards, the two data being functionally
independent; and, on the side of the detector, that the residual
statistic of a context-seeded watermark is proportional to the number of
positions whose seeding window survived intact.

It is mesmerizing, and we cannot help but observe it, how neatly the two
halves of the picture fit. What a meaning-preserving chain preserves is
an invariant; what a watermark occupies is the complement of that
invariant; and the geometric object which measures the complement --- a
holonomy in $\SO(n-1)$ --- turns out on the two-sphere to be nothing
more exotic than the area swept by the path. That the same construction
which was proposed as a way toward pre-verbal thought should also
measure the erosion of a provenance mark is, in our humble opinion, a
point in favour of the construction.

Three limitations are ours to state. The intact-window identity assumes
substitutions that preserve length and a mean-field hypothesis on the
strength of the mark; in real text, where entropy varies with content,
it is an upper bound. The holonomy energy $\eta$ is a lossy scalar
reduction of $H$, and a null result for the former is not a null result
for the latter. And the empirical layer of this article is bounded by
the compute available to it: a small generator, one attack family,
one embedder.

It would be then definitely interesting to analyse whether a
\emph{holonomy-aware} watermark can be constructed --- one whose seeding
depends on a quantity invariant under the transport, rather than on the
raw token window, and which would therefore be insensitive to the
arrangement of the edits in the way that Proposition~\ref{prop:arrangement}
shows the present schemes are not. The semantic schemes of
\cite{Hou24a} are a first step in that direction, and it would be
natural to evaluate them with the instrument developed here rather than
with an endpoint similarity. In a forthcoming work we intend to study
the holonomy of chains produced by learned paraphrasers rather than by
translation, where the path is chosen adversarially and the geometry
should be correspondingly richer. All in all, we have presented evidence
that the robustness of a watermark is a functional of the path and not
of its endpoint, and that the natural language in which to say so is
that of parallel transport.

\section*{Use of Generative AI}

The implementation of the reproducibility certificates listed in
Section~\ref{sec:repro} was supported by AI-assisted code generation
(Claude Opus, Anthropic) under the author's direction; all certificates
were independently inspected, run, and validated by the author. AI tools
were also used for language editing during manuscript preparation. The
author conceived the mathematical content, designed and verified the
proofs, and takes full responsibility for the content of this article.

\section*{Acknowledgments}


\end{document}